\documentclass[a4paper]{article}

\usepackage{a4wide}

\usepackage{amsmath}
\usepackage{amsthm}
\usepackage{amssymb}

\usepackage{tikz} 
\tikzstyle{line} = [draw, -latex']

\usetikzlibrary{shapes.geometric,shapes.misc, positioning}
\usetikzlibrary{shapes,arrows}
\usetikzlibrary{backgrounds}

\usepackage{framed} 

\newtheorem{example}{Example}
\newtheorem{definition}{Definition}
\newtheorem{proposition}{Proposition}

\newcommand{\qw}[1] {{\tt {#1}}}
\newcommand{\xa}{\qw{A}}
\newcommand{\xb}{\qw{B}}
\newcommand{\xc}{\qw{C}}
\newcommand{\xd}{\qw{D}}
\newcommand{\xe}{\qw{E}}

\newcommand{\arcs}{{\sf Arcs}}
\newcommand{\nodes}{{\sf Nodes}}

\newcommand{\args}{{\sf Args}}  
  
\newcommand{\codomain}{{\sf Codomain}}

\newcommand{\inform}{{\sf Inform}}

\newcommand{\graph}{{\cal G}}
\newcommand{\lab}{{\cal L}}

\newcommand{\inst}{{\cal I}}

\newcommand{\igraph}{{\cal X}}

\newcommand{\support}{{\sf S}}
\newcommand{\claim}{{\sf C}}

\begin{document}

\title{\noindent\rule{\textwidth}{1.5pt}\\
Some Options for Instantiation of Bipolar Argument Graphs\\ 
with Deductive Arguments\\
\noindent\rule{\textwidth}{1.5pt}}

\author{
Anthony Hunter\\
Department of Computer Science\\ University College London\\ 
London, UK\\
({\tt anthony.hunter@ucl.ac.uk})}

\maketitle


\begin{abstract}
Argument graphs provide an abstract representation of an argumentative situation. A bipolar argument graph is a directed graph  where each node  denotes an argument, and each arc denotes the influence of one argument on another. Here we assume that the influence is supporting, attacking, or ambiguous. 
In a bipolar argument graph, each argument is atomic and so it has no internal structure. Yet to better understand the nature of the individual arguments, and how they interact, it is important to consider their internal structure.
To address this need, this paper presents a framework based on the use of logical arguments to instantiate bipolar argument graphs, and a set of possible constraints on instantiating arguments that take into account the internal structure of the arguments, and the types of relationship between arguments. 
\end{abstract}

\section{Introduction}

Bipolar argumentation is a generalization of abstract argumentation that incorporates a support relation in addition to the attack relation \cite{CayrolLS05,CayrolLS05b,Amgoud2008,CayrolLS13}. 
A bipolar argument graph is a directed graph where each node  denotes an argument, and each arc denotes the influence of one argument on another.
The label denotes the type of influence with options including positive (supporting) and negative (attacking).
In order to determine the acceptable arguments from a biploar argument graph, dialectical semantics can be generalized to handle both support and attack relations \cite{CayrolLS13}. Alternatively, the approach of abstract dialectical frameworks \cite{Brewka2010,Brewka2014}, gradual semantics \cite{Amgoud2013,AmgoudBNDV17,Bonzon16,CayrolLS05b,LeiteMartins11,Rago16,Costa-Pereira:2011,BRTAB15,Potyka18,Pu2014,PuLZL15,Baroni2019}, or epistemic graphs \cite{Hunter2020} can be used.

An issue with bipolar argument graphs is that each node is an abstract argument, and so the meaning of the argument is unspecified. To address this issue, we investigate how nodes can be instantiated with deductive arguments. By doing this, we can systematically investigate the nature of the contents (premises and claim) and their interplay with the structure of the graph.  
But this then raises questions about what kinds of instantiations are appropriate and what they  mean. It also raises questions about how we can compare instantiated bipolar argument graphs to show, for instance, that two instantiations are equivalent, and how we can investigate the interplay of the premises of an argument and the claims of the arguments that support or attack it.

So the aim of this paper is to provide a framework for instantiating bipolar argument graphs. We will consider a range of constraints that capture potentially important restrictions that we may wish to impose on instantiations. Using these constraints, we can then investigate various options for instantiating bipolar argument graphs. 

We assume that we start with a bipolar argument graph (for example, it might have been given to us, 
or we might have generated it using some argument mining methods), 
and we want to better understand what the arguments in the graph mean. By instantiating the abstract arguments, we make explicit what the premises and claims are, and how each argument relates to the others. It is a form of commitment. As part of this process, we would be constructing the premises that are used in the arguments, and so we would be constructing the knowledgebase on which the arguments are based.

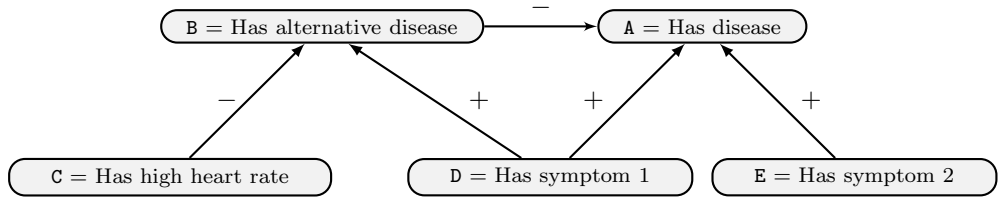
\begin{figure}[t]
\centering
\begin{tikzpicture}
[->,>=latex,thick, arg/.style={draw,text centered, 
shape=rectangle, rounded corners=6pt,
fill=gray!10,font=\footnotesize}]
\node[arg] (a) [text width=25mm] at (8,2) {$\xa$ = Has disease};
\node[arg] (b) [text width=40mm] at (3,2)  {$\xb$ = Has alternative disease};
\node[arg] (d) [text width=35mm] at (6,0) {$\xd$ = Has symptom 1};
\node[arg] (e) [text width=35mm] at (10,0) {$\xe$ = Has symptom 2};
\node[arg] (c) [text width=40mm] at (1,0)  {$\xc$ = Has high heart rate };
\path[line]	(b) edge[] node[above] {$-$} (a);
\path	(d) edge node[left,xshift=-5pt] {$+$} (a);
\path	(e) edge node[right,xshift=5pt] {$+$} (a);
\path	(c) edge node[left] {$-$} (b);
\path	(d) edge node[right,xshift=9pt] {$+$} (b);
\end{tikzpicture}
\caption{\label{fig:epilogic1}Example of a bipolar argument graph concerning diagnosis of a disease based on belief in symptoms and a differential diagnosis which in turn is based on whether the patient has a high heart rate. 
Each argument is a claim with implicit support. 
The $+$ (respectively $-$) label denotes support (respectively attack) relations. 
}
\end{figure}

We proceed as follows:
In Section \ref{section:deductive}, we review deductive argumentation;
In Section \ref{section:instantiation}, we define instantiations of bipolar argument graphs, and illustrate with motivating examples; 
In Section \ref{section:constraints}, we consider a set of constraints that specify desirable properties for well-behaved instantiations of bipolar argument graphs;
And in Section \ref{section:discussion}, we discuss this proposal and some possibilities for future work.

\section{Deductive argumentation}
\label{section:deductive}

We briefly review deductive argumentation \cite{Cayrol1995,Besnard2001,Besnard2008,Gorogiannis2011}.
We consider a classical propositional or first-order language with the classical consequence relation denoted by the $\vdash$ relation.
We use $\alpha, \beta, \gamma, \ldots$ to denote formulae and 
$\Delta, \Phi, \Psi, \ldots$ to denote sets of formulae.
For the following definitions, we first assume a knowledgebase $\Delta$ 
(a finite set of formulae) and use this $\Delta$ throughout for the knowledge for instantiating arguments.

For a set of formulae $\Phi$, let ${\sf Cn}(\Phi)$ be the {\bf consequence closure} of $\Phi$
(i.e. ${\sf Cn}(\Phi)$ = $\{ \psi \mid \Phi\vdash\psi \}$). 
Sets of formulae $\Phi$ and $\Psi$ are {\bf equivalent sets of formulae},
denoted $\Phi\equiv\Psi$,
iff ${\sf Cn}(\Phi)$ = ${\sf Cn}(\Psi)$.
Formulae $\phi$ and $\psi$ are {\bf equivalent formulae},
denoted $\phi\equiv\psi$,
iff ${\sf Cn}(\{\phi\})$ = ${\sf Cn}(\{\psi\})$.

\begin{figure}[t]
\centering
\begin{tikzpicture}
[->,>=latex,thick, arg/.style={draw,align=left,
shape=rectangle, rounded corners=6pt,
fill=gray!10,font=\footnotesize}]
\node[arg] (a) [text width=100mm] at (0,7) 
{
$\inst(\xa)$\\ 
$\tt has\_symptom1(continuous,noticable))$,\\ 
$\tt \neg has\_alternative\_disease$,\\
$\tt (has\_symptom1(frequent,pronounced) \vee has\_symptom1(continuous,noticable))$\\ 
$\tt \hspace{2cm} \wedge has\_symptom2 \wedge \neg has\_alternative\_disease$\\
$\tt \hspace{2cm} \rightarrow has\_disease$\\
.\\
\hrule 
$\tt  HasDisease$
};
\node[arg] (b) [text width=100mm] at (0,1)  
{
$\inst(\xb)$\\ 
$\tt \neg high\_rate(high,chronic)$,\\
$\tt high\_rate(high,chronic) \vee high\_rate(high,occasional)$,\\
$\tt has\_symptom1(continuous,noticable)$, \\
$\tt heart\_rate(high,occasional) \wedge has\_symptom1(continuous,noticable)$\\ 
$\tt \hspace{2cm} \rightarrow has\_alternative\_disease$\\
.\\
\hrule 
$\tt has\_alternativeDisease$
};
\node[arg] (d) [text width=50mm] at (6,4) 
{$\inst(\xd)$\\
$\tt has\_symptom1(continuous,noticable)$\\
.\\
\hrule
$\tt has\_symptom1(continuous,noticable)$\\
};
\node[arg] (e) [text width=25mm] at (-4,4) 
{$\inst(\xe)$\\
$\tt has\_symptom2$\\
.\\
\hrule
$\tt has\_symptom2$\\
};
\node[arg] (c) [text width=50mm] at (0,-2)  
{$\inst(\xc)$\\ 
$\tt heart\_rate(high,chronic)$\\
.\\
\hrule
$\tt heart\_rate(high,chronic)$\\
};
\path[line]	(b) edge[] node[left] {$-$} (a);
\draw[->] (d.north) -- (6,7) -- node[above] {$+$}(a.east);
\draw[->] (e.west) -- (-6,4) -- (-6,7) -- node[above] {$+$}(a.west);
\path	(c) edge node[left] {$-$} (b);
\draw[->] (d.south) -- (6,1) -- node[below] {$+$}(b.east);
\end{tikzpicture}
\caption{\label{fig:instantiation}
An instantiation of the  graph given in Figure \ref{fig:epilogic1}.
Each deductive argument is presented as a set of premises above the line, and the claim below the line.
}
\end{figure}
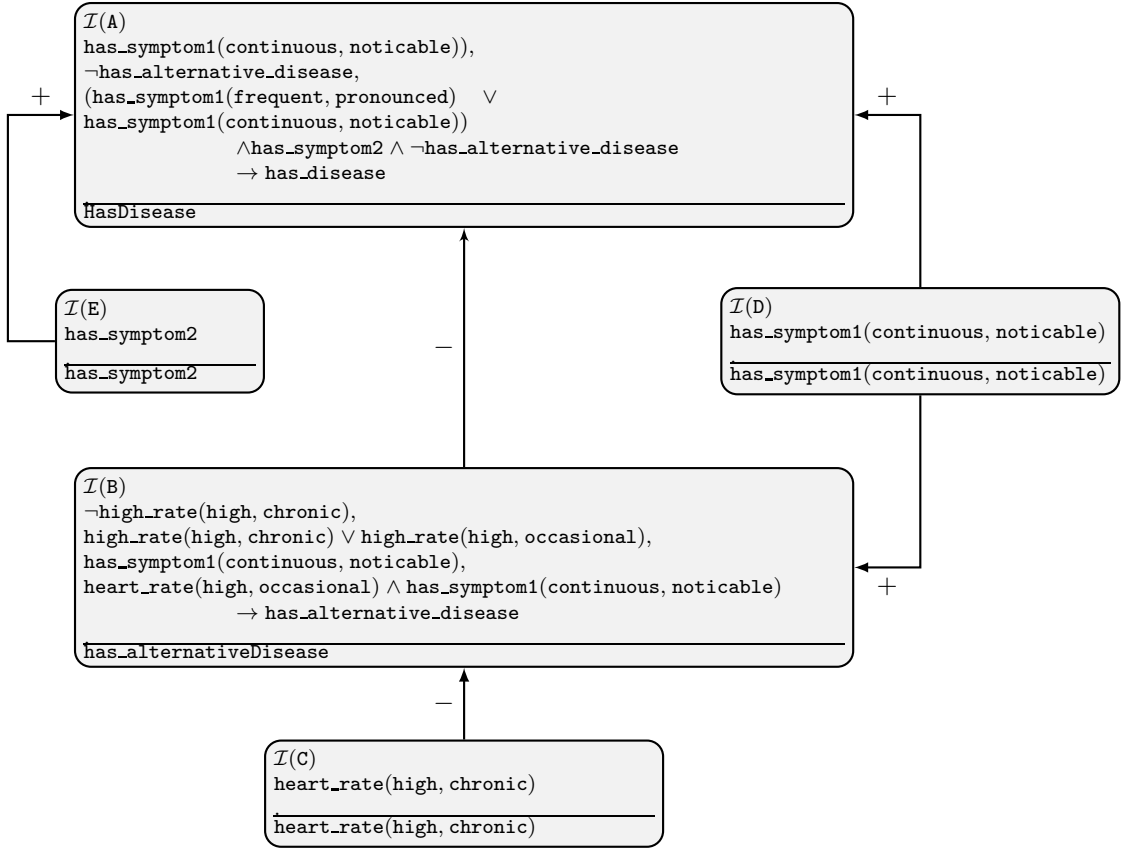

In deductive argumentation, a {\bf deductive argument} is a pair $\langle\Phi,\alpha\rangle$ 
where $\Phi \subseteq {\cal L}$ is a minimal set such that $\Phi$ is consistent and $\Phi$ entails the claim $\alpha$ 
(i.e. $\Phi\vdash\alpha$, $\Phi\not\vdash\bot$, and there is no $\Phi'\subset\Phi$ such that $\Phi'\vdash\alpha$)  \cite{Cayrol1995}.
For a deductive argument $\langle\Phi,\alpha\rangle$,
$\Phi$ is the {\bf support}, or {\bf premises}, of the argument,
and $\alpha$ is the {\bf claim} of the argument.
Also, for $\langle\Phi,\alpha\rangle$,
let $\support(\langle\Phi,\alpha\rangle) = \Phi$, 
and let $\claim(\langle\Phi,\alpha\rangle) = \alpha$.
For arguments $\langle\Phi,\alpha\rangle$
and $\langle\Psi,\beta\rangle$,
they are {\bf equivalent arguments}, denoted $\langle\Phi,\alpha\rangle\equiv\langle\Psi,\beta\rangle$,
iff $\Phi\equiv\Psi$ and $\alpha\equiv\beta$.

We use $A$, $B$, $C$, $\ldots$ to denote abstract arguments,
and we use $I$, $J$, $K$, $\ldots$ to denote deductive arguments.
We use teletype font for these symbols in the examples.
Also, we use lower case letters for propositional and predicate formulae in the examples.

We have a range of options for the definition of {\bf counterargument} (taken from \cite{Besnard2001,Gorogiannis2011}) including the following
where $I$ and $J$ are deductive arguments: 
$I$ is a {\bf defeater} of $J$
if ${\claim}(I) \vdash\neg \bigwedge {\support}(J)$; 
$I$ is a {\bf undercut} of $J$ if there exists a
$\Psi \subseteq {\support}(J)$ 
s.t. ${\claim}(I) \equiv \neg \bigwedge \Psi$; 
$A$ is a {\bf direct undercut} of $J$ if there exists a
$\alpha \in {\support}(J)$ s.t. ${\claim}(I) \equiv \neg \alpha$; 
$A$ is a {\bf canonical undercut} of $B$ if 
${\sf Claim}(A) \equiv \neg \bigwedge {\sf Support}(B)$; 
$I$ is a {\bf defeating rebuttal} of $J$ if 
${\claim}(I) \vdash \neg {\claim}(J)$.

\begin{example}
From the knowledgebase $\Delta = \{ \tt
	a\vee b, a \leftrightarrow b,
	c\rightarrow a,
	\neg a\wedge \neg b,
	a,b,c,
	a \rightarrow b,
	\neg a, \neg b, \neg c
	\}$, 
the arguments and counterarguments include the following.
\[
\begin{array}{l}
\langle \{\tt  a \vee b, c \}, (a \vee b) \wedge c \rangle
	\mbox{ is a defeater of } 
	\langle \{ \neg a, \neg b \}, \neg a \wedge \neg b \rangle\\
	

\langle \{\tt \neg a\wedge \neg b \}, \neg (a \wedge b) \rangle
	\mbox{ is an undercut of } 
	\langle \{ a,b,c \}, a \wedge b \wedge c \rangle\\
	
\langle \{\tt \neg a\wedge \neg b \}, \neg a \rangle
	\mbox{ is a direct undercut of } 
	\langle \{ a,b,c \}, a \wedge b \wedge c \rangle\\

\langle \{\tt \neg a\wedge \neg b \}, \neg (a \wedge b \wedge c) \rangle
	\mbox{ is a canonical undercut of } 
	\langle \{ a,b,c \}, a \wedge b \wedge c \rangle\\

	
\langle \{\tt a, a \rightarrow b \}, b  \rangle
	\mbox{ is a defeating rebuttal of } 
	\langle \{ \neg a \wedge \neg b, \neg c \}, \neg(b \vee c) \rangle\\
\end{array}
\]

\end{example}

 Note, the definitions presented in this section can also be used directly with first-order classical logic, so $\Delta$ and $\alpha$ can be from the first-order classical language. 
 
For further coverage of the properties of deductive argumentation, see \cite{Besnard2001,Besnard2008,Gorogiannis2011}, 
and for a review of instantiation of argument graphs with deductive arguments, see \cite{BesnardHunter2014}.

\section{Instantiation of bipolar argument graphs}
\label{section:instantiation}

We now briefly review bipolar argument graphs,
and then we bring the approaches of bipolar argument graphs and of deductive argumentation together using an instantiate function that assigns a deductive argument to each node in the bipolar argument graph.

In this paper, we assume a {\bf bipolar argument graph} is a tuple $(\graph,\lab)$
where $\graph$ is a directed graph and $\lab$ is an assignment of a label to each arc.
See Figures \ref{fig:epilogic1}, \ref{ex:saltsweet}, \ref{ex:seaside}, 
and \ref{ex:threecycle} for some examples of bipolar argument graphs.
Let $\nodes(\graph)$ be the nodes in the graph, where each node denotes an argument,
and let $\arcs(\graph)$ be the arcs in the graph, where each arc denotes the first argument in the pair having an impact on the second argument. 
So the labelling function $\lab$ is an assignment from $\nodes(\graph)$ to $\{+,-,\ast\}$
where $+$ denotes the source argument supports the target argument,
$-$ denotes the source argument attacks the target argument,
and $\ast$ denotes the influence of the source argument on the target argument is ambiguous (i.e. it could be either supporting or attacking).

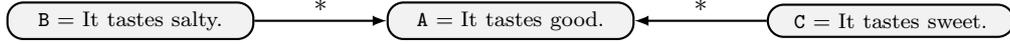
\begin{figure}
\begin{center}
\begin{tikzpicture}
[->,>=latex,thick, arg/.style={draw,text centered, 
shape=rectangle, rounded corners=6pt,
fill=gray!10,font=\footnotesize}]
\node[arg] (a1) [text width=30mm] at (5,0) {$\xa$ = {\rm It tastes good.}};
\node[arg] (a2) [text width=30mm] at (0,0)  {$\xb$ = {\rm It tastes salty.}};
\node[arg] (a3) [text width=30mm] at (10,0) {$\xc$ = {\rm It tastes sweet.}};
\path	(a2) edge node[above] {$*$} (a1);
\path	(a3) edge node[above] {$*$} (a1);
\end{tikzpicture}
\end{center}
\caption{\label{ex:saltsweet}
The bipolar argument graph $(\graph,\lab)$ for the taste of a food item.
Each argument is a claim with implicit premises. 
We explain these labels as follows. 
Consider an item of food (apart from chocolate or caramel). 
If it tastes salty and it does not taste sweet, 
or it does not taste salty and it tastes sweet,
then it tastes ok,
and if tastes salty and sweet,
then it does not taste ok. 
Given these assumptionms, the influence of $\qw{B}$ and $\qw{C}$ on $\qw{A}$ is not simply a positive or a negative one.
Rather it is ambiguous. 
}
\end{figure}

\begin{figure}
\begin{center}
\begin{tikzpicture}
[->,>=latex,thick, arg/.style={draw,text centered, 
shape=rectangle, rounded corners=6pt,
fill=gray!10,font=\footnotesize}]
\node[arg] (a1) [text width=50mm] at (0,0) {$\xa$ = {\rm A holiday by the sea-side is good because water-based activities are fun.}};
\node[arg] (a2) [text width=50mm] at (8,0)  {$\xb$ = {\rm Going on a trip on an offshore yacht can be challenging.}};
\path	(a2) edge node[above] {$*$} (a1);
\end{tikzpicture}
\end{center}
\caption{\label{ex:seaside}A bipolar argument graph where $\xb$ is ambigious as to whether it is supporting or attacking. It could be supporting if ``being challenging" is interpreted as a ``fun water-based activity" or it could be attacking if ``being challenging" is interpreted as a ``water-based activity that is not fun".}
\end{figure}
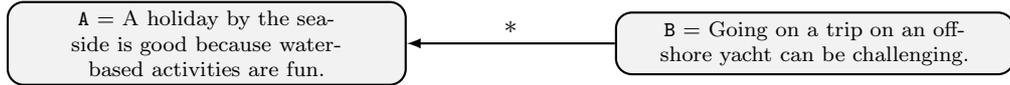

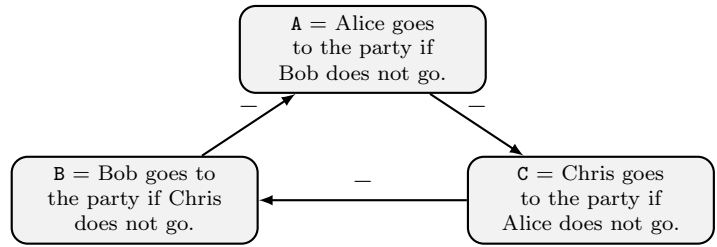
\begin{figure}
\begin{center}
\begin{tikzpicture}
[->,>=latex,thick, arg/.style={draw,text centered, 
shape=rectangle, rounded corners=6pt,
fill=gray!10,font=\footnotesize}]
\node[arg] (a1) [text width=30mm] at (3,2) {$\xa$ = {\rm Alice goes to the party if Bob does not go.}};
\node[arg] (a2) [text width=30mm] at (0,0)  {$\xb$ = {\rm Bob goes to the party if Chris does not go.}};
\node[arg] (a3) [text width=30mm] at (6,0) {$\xc$ = {\rm Chris goes to the party if Alice does not go.}};
\path	(a2) edge node[above] {$-$} (a1);
\path	(a3) edge node[above] {$-$} (a2);
\path	(a1) edge node[above] {$-$} (a3);
\end{tikzpicture}
\end{center}
\caption{\label{ex:threecycle}
The bipolar argument graph $(\graph,\lab)$ that involves a three-cycle.
}
\end{figure}

Using the definition of bipolar argument graph, we can add the notion of an instantiation function. This is a function that assigns a deductive argument to each node in the graph.

\begin{definition}
An {\bf instantiation function} is a function $\inst:\nodes(\graph)\rightarrow {\cal D}$
where ${\cal D}$ is a set of deductive arguments. 
An {\bf instantiated bipolar argument graph} is a tuple $(\graph,\lab,\inst)$ where $(\graph,\lab)$ is a bipolar argument graph and $\inst$ is an instantiation function for $\nodes(\graph)$. 
\end{definition}

We illustrate instantiation functions for bipolar argument graphs in Figure \ref{fig:instantiation} above
and in Examples \ref{ex:threecycle2} and \ref{ex:saltsweetagain} below.

\begin{example}
\label{ex:threecycle2}
Continuing the example in Figure \ref{ex:threecycle}, 
let $\qw{a}$ = {\rm Ann goes to the party}, 
$\qw{b}$ = {\rm Bob goes to the party},
and 
$\qw{c}$ = {\rm Chris goes to the party}.
Let $\inst(\xa) = \langle \{ \tt \neg b, \neg b \rightarrow a \}, a \rangle$.
and $\inst(\xb) = \langle \{ \tt \neg c, \neg c \rightarrow b \}, b  \rangle$ 
and $\inst(\xc) = \langle \{ \tt \neg a, \neg a \rightarrow c \}, c  \rangle$.
\end{example}

\begin{example}
\label{ex:saltsweetagain}
Continuing the example in Figure \ref{ex:saltsweet}, 
let $\qw{sa}$ = {\rm it tastes salty}, 
$\qw{sw}$ = {\rm it tastes sweet},
and 
$\qw{tg}$ = {\rm it tastes good}.
Let $\inst(\xa) = \langle \{ \tt sa \vee sw, \neg sa \vee \neg sw, 
			(sa \wedge \neg sw) \vee (\neg sa \wedge sw) \rightarrow tg  \}, tg \rangle$.
and $\inst(\xb) = \langle \{ \tt  sa \},  sa \rangle$ 
and $\inst(\xc) = \langle \{ \tt  sw \},  sw \rangle$.
\end{example}

The above example shows how there are multiple instantiations that would make sense given the text associated with the abstract arguments. This is even with the same set of propositional atoms in the logical language.
For instance, for $\qw{A}$, then $\inst(\xa) = \langle \{ \tt sa, \neg sw, (sa \wedge \neg sw) \rightarrow tg  \}, tg \rangle$,
and $\inst(\xa) = \langle \{ \tt \neg sa, sw, (\neg sa \wedge sw) \rightarrow tg  \}, tg \rangle$, are also possible instantiations.

The observation that there can be multiple instantiations is a reflection of 
the ambiguity that arises from abstract argumentation where the exact meaning of the argument is not formally specified.
To further illustrate how there are choices for how we instantiate abstract arguments, we consider the following examples. 

\begin{example}
\label{ex:instantiate1}
Consider the bipolar argument graph in Figure \ref{fig:epilogic1}.
We can define the instantiation function below
where the leaf nodes are atomic arguments,
and the non-leaf nodes are based on a conditional formula that explicitly takes the supporting arguments, 
but not the attacking arguments, into account. 
\[
\begin{array}{l}
\inst(\xa) = \langle\footnotesize \tt \{has\_symptom1, has\_symptom2,\\
\hspace{15mm}\tt has\_symptom1 \wedge has\_symptom2 \rightarrow has\_disease\}, has\_disease\rangle\\
\inst(\xb) = \langle\footnotesize\tt \{has\_symptom1, has\_symptom1 \rightarrow has\_alternative\}, has\_alternative\rangle\\
\inst(\xc) = \langle\footnotesize\tt \{has\_high\_heart\_rate\}, heart\_high\_heart\_rate\rangle\\
\inst(\xd) = \langle\footnotesize\tt \{has\_symptom1\}, has\_symptom1\rangle\\
\inst(\xe) = \langle\footnotesize\tt \{has\_symptom2\}, has\_symptom2\rangle\\
\end{array}
\]
\end{example}

An alternative to the above example is the following example 
where the non-leaf nodes are based on a conditional formula that explicitly takes both the supporting arguments, 
and the attacking arguments, into account.

\begin{example}
\label{ex:instantiate2}
Consider the bipolar graph given in Figure \ref{fig:epilogic1}. 
We can define the instantiation function as follows
where the leaf nodes are atomic arguments,
and the non-leaf nodes are based on a conditional formula that takes the supporting and attacking arguments into account. 
\[
\begin{array}{l}
\inst(\xa) = \langle\tt \{has\_symptom1, has\_symptom2, \neg has\_alternative,\\
\hspace{15mm}\tt has\_symptom1 \wedge has\_symptom2 \wedge \neg has\_alternative \rightarrow has\_disease\}, has\_disease\rangle\\
\inst(\xb) = \langle\tt \{has\_symptom1, \neg has\_high\_heart\_rate, \\
\hspace{15mm}\tt has\_symptom1 \wedge \neg has\_high\_heart\_rate \rightarrow has\_alternative\}, has\_alternative\rangle\\
\inst(\xc) = \langle\tt \{has\_high\_heart\_rate\}, has\_high\_heart\_rate\rangle\\
\inst(\xd) = \langle\tt \{has\_symptom1\}, has\_symptom1\rangle\\
\inst(\xe) = \langle\tt \{has\_symptom2\}, has\_symptom2\rangle\\
\end{array}
\]
\end{example}

Another alternative to Example \ref{ex:instantiate1} and Example \ref{ex:instantiate2} 
is to give more sophisticated premises as we do in Figure \ref{fig:instantiation}
or in Example \ref{ex:predicate} below. 

\begin{example}
\label{ex:predicate}
Consider argument $\xc$ in the bipolar argument graph given in Figure \ref{fig:epilogic1}.
We could define $\inst$ so that $\support(\inst(\xc))$ is the following set which includes a quantified formula from first-order predicate logic.
\[
\begin{array}{l}
\qw{blood\_pressure(12Oct21,123)},\\
\qw{blood\_pressure(12Nov21,127)},\\
\tt \forall x_1,x_2,y_1,y_2\; 
\qw{blood\_pressure(x_1,y_1)} \wedge 
\qw{blood\_pressure(x_2,y_2)} \\
\hspace{1cm} 
\tt \wedge \; x_1 \neq x_2  
\wedge (y_1 \geq 105) 
\wedge (y_2 \geq 105) 
\rightarrow heart\_rate(high,chronic)
\end{array}
\]
\end{example}

We can assume that a bipolar argument graph is given to us (e.g. we could be listening to a discussion), and we want to check that the instantiated bipolar argument graph is reasonable. So this then raises the question of what ``reasonable" means in this context. We address this question in the next section by considering appropriate constraints on instantiations.

\section{Constraints on instantiation}
\label{section:constraints}

In this section, we define constraints that specify options for saying whether an instantiated bipolargraph $\igraph = (\graph,\lab,\inst)$ is reasonable. 
We will consider two types of constraints that specify types of instantiation. 
These are
{\bf knowledgebase constraints} that connect the arguments in the bipolar argument graph with a knowledgebase 
that is used to provide the premises for the arguments, 
and {\bf structural constraints} that connect the structure of the bipolar argument graph (i.e. the nodes and arcs in $\graph$ and its labelling $\lab$)
with the logical nature of the instantiated arguments.

\subsection{Knowledgebase constraints}

We can consider the knowledge used in the instantiated arguments as coming from a knowledgebase. 
For a knowledgebase $\Delta$, let $\args(\Delta)$ be the set of deductive arguments from $\Delta$.

\begin{definition}
For an instantiated bipolar argument graph $\igraph = (\graph,\lab,\inst)$ 
and a knowledgebase $\Delta$
\begin{itemize}

\item $\igraph$ {\bf uses} $\Delta$
iff $\codomain(\igraph) \subseteq \args(\Delta)$

\item $\igraph$ {\bf exhausts} $\Delta$
iff $\codomain(\igraph) = \args(\Delta)$

\end{itemize}
where $\codomain(\igraph)$ is the set of instantiated arguments in $\igraph$ (i.e. 
$\codomain(\igraph) = \{ \inst(A) \mid A \in \nodes(\graph) \} $).
\end{definition}

Also using the $\codomain$ function, 
we can define the $\inform$ function to retrieve the formulae used in the premises of the arguments 
(i.e. $\inform(\igraph) = \{ \phi \mid \phi \in \support(I) \mbox{ and } I \in \codomain(\igraph) \}$. 
So $\igraph$ uses $\Delta$
when all the formulae used in the premises of the instantiated arguments come from the knowledgebase.
(i.e. $\inform(\igraph)$). 
We say that $\igraph$ {\bf displays} $\Delta$ iff $\Delta = \inform(\igraph)$.
In other words, each formulae in $\Delta$ is used as a premise in at least one argument in the instantiated bipolar argument graph.

\begin{proposition}
For an instantiated bipolar argument graph $\igraph$
and knowledgebase $\Delta$, 
if $\igraph$ exhausts $\Delta$,
then $\igraph$ displays $\Delta$.
\end{proposition}

\begin{proof}
Assume $\igraph$ exhausts $\Delta$.
So $\codomain(\igraph) = \args(\Delta)$.
So for each $\alpha \in \Delta$, there is a deductive argument $I \in \args(\Delta)$ such that $\alpha\in\support(I)$,
and hence there is a deductive argument $I \in \codomain(\igraph)$ such that $\alpha\in\support(I)$.
So $\Delta = \inform(\igraph)$.
So $\igraph$ displays $\Delta$.
\end{proof}

Obviously, it is an extreme situation when $\igraph$ exhausts $\Delta$, and it is straightforward to have less extreme situations where $\igraph$ displays $\Delta$.

We can satisfy the uses relation if the knowledgebase is the set of formulae in the premises of the instantiated arguments. However, in general, we do not expect the exhausts relation holds: With classical logic, there are many deductive arguments that can be constructed from a small set of formulae, and there is a lot of repetition between the deductive arguments. For instance, suppose we have the deductive arguments $\langle \{  \tt a, \neg a \vee \neg b \}, \neg b \rangle$ and $\langle \{  \tt b \}, b \rangle$, is there any value of having further deductive arguments such as $\langle \{  \tt \neg a \vee \neg b, b \}, \neg a \rangle$? It just is another way of showing that there is an inconsistency involving these three formulae. Bringing this third deductive argument up in a discussion or debate is unlikely to be appreciated given that it is in a sense redundant. So we include the exhausts relation for understanding the space of constraints but we do not advocate its use in general.

\subsection{Structural constraints}

Next, we consider some potentially desirable constraints that we might want to assume on instantiating a bipolar argument graph. 
Different constraints give rise to different choices of deductive arguments and relationships between them.

\subsubsection{Definitions for structural constraints}

We will consider five sets of constraints called EQUIV, NEG, POS, INC, and SUP.
For each constraint, we assume that the constraint concerns all the arguments in an instantiated bipolar argument graph 
$\igraph = (\graph,\lab,\inst)$, 
and so the constraint holds for a graph if and only if it holds for all the arguments in the graph.

\begin{figure}
\begin{center}
\begin{tikzpicture}
[->,>=latex,thick, arg/.style={draw,text centered, 
shape=rectangle, rounded corners=6pt,
fill=gray!10,font=\footnotesize}]
\node[arg] (a) [text width=40mm] at (5,0) {$\inst(\xa) = \langle \tt \{ f, f \rightarrow e \}, e \rangle $};
\node[arg] (b) [text width=40mm] at (0,0) {$\inst(\xb) = \langle \tt \{ p, p \rightarrow \neg f \}, \neg f \rangle $};
\node[arg] (c) [text width=40mm] at (10,0) {$\inst(\xc) = \langle \tt \{ c, c \rightarrow f \}, f \rangle $};
\path	(b) edge[] node[above] {$-$} (a);
\path	(c) edge[] node[above] {$+$} (a);
\draw[->] (c) |-  (5,0.8) node[above] {$-$} -| (b) ;
\draw[->] (b) |-  (5,-0.8) node[below] {$-$} -| (c) ;
\end{tikzpicture}
\end{center}
\caption{\label{fig:neg:1}This instantiated bipolar argument graph (where $\tt e$ = {\em it escapes predators by flying}, $\tt f$ = {\em it is capable of flying}, $\tt p$ = {\em it is a penguin}, and $\tt c$ = {\em it chirps})  
conforms to {\rm NEG1}, {\rm NEG2}, {\rm NEG3}, {\rm NEG4}, {\rm NEG5}, {\rm NEG6}, {\rm NEG7},
{\rm  POS1}, {\rm  POS2}, {\rm  POS3}, {\rm  POS4}, {\rm  POS6}, {\rm  POS7}, {\rm  POS8}, {\rm  POS9}, but not to {\rm  POS5}.}
\end{figure}
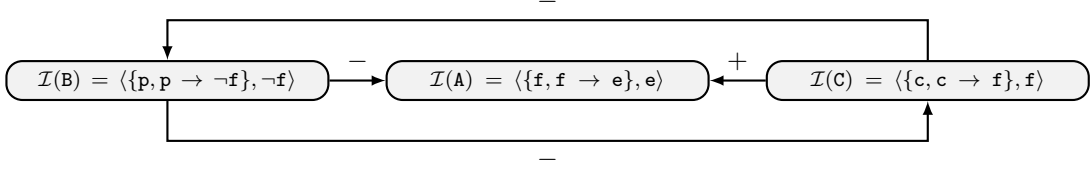

The first set of constraints just has the constraint EQUIV. This is a syntax-independence requirement and so equivalent pairs of arguments have the same labels.
\[
\mbox{(EQUIV) } \text{ if } \inst(A) \equiv \inst(A') \text{ and } \inst(B) \equiv \inst(B'), \text{ then } \lab(A,B) = \lab(A',B')\\
\]

Next we consider the NEG constraints. These limit what the negative label implies, and thereby limit what are allowable instantiations for negative labelled arcs. 
\[
\begin{array}{l}
\mbox{(NEG1) } \text{ if } \lab(A, B) = -, 
\text{ then }\{ \claim(\inst(A)) \} \not\vdash \claim(\inst(B))\\
\mbox{(NEG2) } \text{ if } \lab(A, B) = -, 
\text{ then }\{ \claim(\inst(A)) \} \not\vdash\phi 
\mbox{ for any } \phi \in \support(\inst(B))\\
\mbox{(NEG3) } \text{ if } \lab(A, B) = -, 
\text{ then }\{ \claim(\inst(A)) \} \cup \support(\inst(B)) \vdash \bot\\
\mbox{(NEG4) } \text{ if } \lab(A, B) = -, \text{ and } \claim(\inst(C)) \vdash \claim(\inst(A)), 
\text{ then } \lab(C, B) = -\\
\mbox{(NEG5) } \text{ if } \lab(A, B) = -, \text{ and } \support(\inst(B)) \subseteq \support(\inst(C)), 
\text{ then } \lab(A, C) = -\\
\mbox{(NEG6) } \text{ if } \lab(A, B) = -, \text{ and } \claim(\inst(C)) \vdash \claim(\inst(A)), 
\mbox{ and } (C,B)\in\arcs(\graph), 
\text{ then } \lab(C, B) \neq + \\
\mbox{(NEG7) } \text{ if } \lab(A, B) = -, \text{ and } \support(\inst(B)) \subseteq \support(\inst(C)), 
\mbox{ and } (A,C)\in\arcs(\graph), 
\text{ then } \lab(A, C) \neq + \\
\end{array}
\]

We explain these constraints as follows.
{\rm NEG1} and {\rm NEG2} ensure that the claim of the attacking argument does not imply the claim or support of the attacked argument.
Next {\rm NEG3} and {\rm NEG4} (respectively {\rm NEG5}) have been adapted 
from postulates by Gorogiannis and Hunter \cite{Gorogiannis2011} (respectively Amgoud and Besnard \cite{Amgoud2009})
for notions of counterargument in logic-based argumentation. 
{\rm NEG3} mandates that if an argument attacks
another, then the claim of the former should be inconsistent with the
support of the latter. 
This  reflects a fundamental assumption in
logical argumentation that for an attack to take place,
the attacking argument must make it specific in its claim that it contradicts
the evidence offered by the attacked argument.
{\rm NEG4} and {\rm NEG5} impose certain fairness restrictions on existing attacks:
{\rm NEG4} requires that any argument with a stronger claim than $A$, i.e., one that
logically entails that of $A$, should also attack anything $A$ attacks.
{\rm NEG5} mandates that any argument whose
support is a superset of that of $B$, and thus stronger than that of $B$, should
also be attacked by $A$. 
We also consider weaker versions of {\rm NEG4} and {\rm NEG5}. 
For {\rm NEG6}, there is the extra condition that there is an arc from $C$ to $B$, and if so, then the label cannot be $+$.
This means that if $\lab(A, B) = -$ and $\claim(\inst(C)) \vdash \claim(\inst(A))$ hold, 
then it is not necessarily the case that there is an arc from $C$ to $B$ with label $-$.
Similarly for {\rm NEG7}, if  $\lab(A, B) = -$ and $\support(\inst(B)) \subseteq \support(\inst(C))$, 
it is not necessarily the case that there is an arc from $C$ to $B$ with label $-$.
Adopting {\rm NEG6} (or {\rm NEG7}) gives the flexibility to include counterarguments with stronger claims (or weaker premises) in the graph but to not have an arc labelled with $-$. This might be because the extra argument is not simply attacking, and so there might be a need to label it with $*$.

We give two simple examples of instantiated bipolar argument graphs in Figures \ref{fig:neg:1} and \ref{fig:neg:2} that satisfy the NEG constraints as well as some of the POS constraints which we define below.

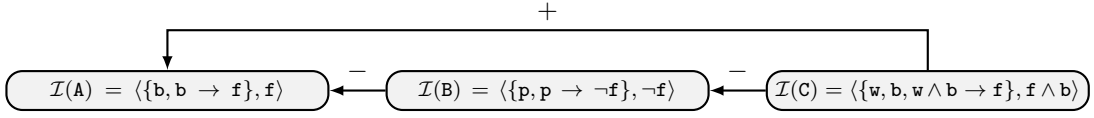
\begin{figure}
\begin{center}
\begin{tikzpicture}
[->,>=latex,thick, arg/.style={draw,text centered, 
shape=rectangle, rounded corners=6pt,
fill=gray!10,font=\footnotesize}]
\node[arg] (a) [text width=40mm] at (0,0) {$\inst(\xa) = \langle \tt \{ b, b \rightarrow f \}, f \rangle $};
\node[arg] (b) [text width=40mm] at (5,0) {$\inst(\xb) = \langle \tt \{ p, p \rightarrow \neg f \}, \neg f \rangle $};
\node[arg] (c) [text width=40mm] at (10,0) {$\inst(\xc) = \langle \tt \{ w, b, w\wedge b \rightarrow f \}, f \wedge b \rangle $};
\path	(b) edge[] node[above] {$-$} (a);
\path	(c) edge[] node[above] {$-$} (b);
\draw[->] (c) |-  (5,0.8) node[above] {$+$} -| (a) ;
\end{tikzpicture}
\end{center}
\caption{\label{fig:neg:2}This instantiated bipolar argument graph (where $\tt b$ = {\em it is a bird}, $\tt f$ = {\em it is capable of flying}, $\tt p$ = {\em it is a penguin}, and $\tt w$ = {\it has wings})  
conforms to {\rm NEG1}, {\rm NEG2}, {\rm NEG3}, {\rm NEG4}, {\rm NEG5}, {\rm NEG6}, {\rm NEG7},
{\rm  POS1}, {\rm  POS2}, {\rm  POS4}, {\rm  POS5}, {\rm  POS6}, {\rm  POS7}, {\rm  POS8}, {\rm  POS9},
but not to {\rm  POS3}. The argument $\inst({\xc})$ is a supporter that has the same claim as argument $\inst({\xa})$ but it has a more specialized support. }
\end{figure}

The third set of constraints is the POS set. These constraints are like the NEG constraints in that they capture some implications of labelling. 
More specifically, the POS constraints limit what the positive label implies, and thereby limit what are allowable instantiations for positive labelled arcs. 
\[
\begin{array}{l}
\mbox{(POS1) } \text{ if } \lab(A, B) = +, 
\text{ then }\support(\inst(A)) \cup \support(\inst(B))\not\vdash\bot\\
\mbox{(POS2) } \text{ if } \lab(A, B) = +, 
\text{ then }\{ \claim(\inst(A)) \} \cup \support(\inst(B))\not\vdash\bot\\
\mbox{(POS3) } \text{ if } \lab(A, B) = +, 
\text{ then there is a }  \phi \in \support(\inst(B))  
\mbox{ s.t. } \claim(\inst(A)) \mbox{ is } \phi\\
\mbox{(POS4) } \text{ if } \lab(A, B) = +, 
\text{ then there is a }  \Gamma\subseteq \support(\inst(B))  
\mbox{ s.t. } \claim(\inst(A)) \vdash \wedge\Gamma\\
\mbox{(POS5) } \text{ if } \lab(A, B) = +, 
\text{ then } \claim(\inst(A)) \vdash \wedge\support(\inst(B)),
\text{ and } \support(\inst(B))\neq\emptyset\\
\mbox{(POS6) } \text{ if } \lab(A, B) = +, 
\text{ and } \claim(\inst(C)) \vdash \claim(\inst(A)), 
\text{ then } \lab(C, B) = +\\
\mbox{(POS7) } \text{ if } \lab(A, B) = +,  
\text{ and } \support(\inst(B)) \subseteq \support(\inst(C)), 
\text{ then } \lab(A, C) = +\\
\mbox{(POS8) } \text{ if } \lab(A, B) = +, 
\mbox{ and } (C,B)\in\arcs(\graph), 
\text{ and } \claim(\inst(C)) \vdash \claim(\inst(A)), 
\text{ then } \lab(C, B) \neq -\\
\mbox{(POS9) } \text{ if } \lab(A, B) = +,
\mbox{ and } (A,C)\in\arcs(\graph), 
\text{ and } \support(\inst(B)) \subseteq \support(\inst(C)), 
\text{ then } \lab(A, C) \neq -\\
\end{array}
\]

{\rm  POS1} and {\rm  POS2} ensure that if an arc is labelled as a supporting relationship,
then the arguments are consistent either with respect to premises 
or with respect to claim of supporter and premises of supportee.
{\rm  POS3} captures a notion of support 
where the claim of the supporting argument implies one of the premises of the supported argument;
{\rm  POS4} captures a notion of support 
where the claim of the supporting argument implies the conjunction of some of the premises of the supported argument;
{\rm  POS5} captures a notion of support 
where the claim of the supporting argument implies the conjunction of the premises of the supported argument; 
{\rm  POS6} and {\rm  POS7} are counterparts to {\rm NEG4} and {\rm NEG5} respectively. 
{\rm  POS6} ensures that if an argument is supported by a supporter, 
then any argument that has a stronger claim, is also labelled as a supporter;
{\rm  POS7} ensures that if an argument is supported by a supporter, 
then any argument that has a superset of the premises is also supported by that supporter. 
We consider weaker versions of {\rm  POS6} and {\rm  POS7} as follows. 
For {\rm  POS8}, 
if there is a positive label on an arc from A to B,
and an argument C s.t.
$\claim(\inst(C)) \vdash \claim(\inst(A))$ holds,
then either there is no arc from C to B or it is not labelled $-$. 
Similarly for {\rm  POS9}, 
if there is a positive label on an arc from A to B,
and an argument C s.t.
$\support(\inst(B)) \subseteq \support(\inst(C))$ holds,
then either there is no arc from A to C or it is not labelled $-$. 
{\rm POS8} and {\rm POS9} are counterparts to {\rm NEG6} and {\rm NEG7}.

\begin{example}
Consider the following arguments. 
If $\lab(A_1,A_2) = +$, then {\rm  POS3} holds for $A_1$ and $A_2$;
If $\lab(A_1,A_2) = +$, and $\lab(A_3,A_2) = +$, then {\rm  POS6} holds for $A_1$, $A_2$, and $A_3$;
And if $\lab(A_5,A_3) = +$, and $\lab(A_5,A_4) = +$, then {\rm  POS7} holds for $A_3$, $A_4$, and $A_5$.
\[
\begin{array}{l}
\inst(A_1) = \langle\{\tt a, \neg a \vee \neg b \vee \neg c  \}, \neg b \vee \neg c\rangle\\
\inst(A_2) = \langle\{\tt \neg b \vee \neg c, \neg b \rightarrow d, \neg c \rightarrow d  \}, d \rangle\\
\inst(A_3) = \langle\{\tt a, a \rightarrow \neg b \wedge \neg c  \}, \neg b \wedge \neg c\rangle\\
\inst(A_4) = \langle\{\tt a, a \rightarrow \neg b \wedge \neg c, d  \}, \neg b \wedge \neg c \wedge d \rangle\\
\inst(A_5) = \langle\{\tt a  \}, a\rangle\\
\end{array}
\]
\end{example}

In the fourth set of constraints, the INC constraints, we consider what would be the necessary labelling in case the instantiated source argument is a defeater of the instantiated target argument. We consider three options (which are not mutually exclusive) for this.
\[
\begin{array}{l}
\mbox{(INC1) } \text{ if } \{ \claim(\inst(A)) \} \cup \support(\inst(B)) \vdash \bot, 
\text{ then } \lab(A, B) = -\\
\mbox{(INC2) } \text{ if } \{ \claim(\inst(A)) \} \cup \support(\inst(B)) \vdash \bot, 
\text{ then } \lab(A, B) \neq + \\
\mbox{(INC3) } \text{ if } \{ \claim(\inst(A)) \} \cup \support(\inst(B)) \vdash \bot 
\mbox{ and } (A,B)\in\arcs(\graph), 
\text{ then } \lab(A, B) \neq + \\
\end{array}
\]

{\rm  INC1} is the converse of {\rm NEG3}. 
However, it is a strong constraint since it says that for each pair of arguments where the former is a defeater of the latter, then there is an arc in the bipolar argument graph from the former to the latter that is labelled with $-$.
We consider weaker forms of {\rm  INC1} called {\rm  INC2} and {\rm  INC3}.
INC2 ensures that if A is a defeater of B, then there is a not a positive arc from A to B (i.e. the arc is labelled with $*$ or $-$), 
and INC3 ensures that if A is a defeater of B, then there is a not a positive arc from A to B (i.e. there is no arc from A to B or there is arc that is labelled with $*$ or $-$). 
So INC2 imposes an arc between the arguments in case of conflict and the label is not $+$, 
whereas INC3 does not impose further arcs, it just ensures that in case there is an arc, it is not labelled $+$.

\begin{example}
The bipolar argument graph $(\graph,\lab)$ is the following labelled graph.
Let $\inst(\xb) = \langle\{\neg\qw{b}\},\neg\qw{b}\rangle$
and $\inst(\xa) = \langle\{\qw{b},\qw{b}\rightarrow\qw{a}\},\qw{a}\rangle$.
So {\rm INC1} fails but {\rm INC2} and {\rm INC3} succeed. 
\begin{center}
\begin{tikzpicture}
[->,>=latex,thick, arg/.style={draw,text centered, 
shape=rectangle, rounded corners=6pt,
fill=gray!10,font=\footnotesize}]
\node[arg] (a1) [text width=10mm] at (3,0) {$\xa$};
\node[arg] (a2) [text width=10mm] at (0,0)  {$\xb$};
\path	(a2) edge node[above] {$*$} (a1);
\end{tikzpicture}
\end{center}
\end{example}

Finally, in the fifth set of constraints, the SUP set, we consider what would be the necessary labelling in case the instantiated source argument is a supporter of the instantiated target argument. We consider six options (which are not mutually exclusive) for this.
\[
\begin{array}{l}
\mbox{(SUP1) } \text{ if there is a }  \phi \in \support(\inst(B))  
\mbox{ s.t. } \claim(\inst(A)) \mbox{ is } \phi, 
\text{ then } \lab(A, B) = +\\
\mbox{(SUP2) } \text{ if there is a }   \emptyset\subset\Gamma\subseteq \support(\inst(B))  
\mbox{ s.t. } \claim(\inst(A)) \vdash \wedge\Gamma
\text{ then } \lab(A, B) = +\\
\mbox{(SUP3) } \text{ if there is a }  \phi \in \support(\inst(B))  
\mbox{ s.t. } \claim(\inst(A)) \mbox{ is } \phi, 
\text{ then } \lab(A, B) \neq -\\
\mbox{(SUP4) } 
\text{ if there is a }  \emptyset\subset \Gamma\subseteq \support(\inst(B))  
\mbox{ s.t. } \claim(\inst(A)) \vdash \wedge\Gamma,
\text{ then }  \lab(A, B) \neq - \\
\mbox{(SUP5) } \text{ if there is a }  \phi \in \support(\inst(B))  
\mbox{ s.t. } \claim(\inst(A)) \vdash \phi
\mbox{ and } (A,B)\in\arcs(\graph), 
\text{ then } \lab(A, B) \neq -\\
\mbox{(SUP6) } \text{ if there is a } \emptyset\subset  \Gamma\subseteq \support(\inst(B))  
\mbox{ s.t. } \claim(\inst(A)) \vdash \wedge\Gamma
\mbox{ and } (A,B)\in\arcs(\graph), 
\text{ then } \lab(A, B) \neq -\\
\end{array}
\]

{\rm  SUP1} is the  converse of {\rm  POS3}. It imposes a positive arc between pairs of arguments 
when the claim of the supporting argument implies a formula in the premises of the supporting argument. 
{\rm  SUP2} is the converse of {\rm  POS4}. It imposes a positive arc between pairs of arguments 
when the claim of the supporting argument implies the conjunction of some of the premises of the supporting argument.
As with {\rm  INC1}, {\rm  SUP1} and {\rm  SUP2} might be problematic. 
They force the graph to include relationships that might over-complicate the presentation, and they force the label to be $+$ 
whereas the label might need to be $*$ given the constraints involving the pair of arguments. 
To address the second of these concerns, 
the {\rm SUP3} and {\rm SUP4} constraints may be more desirable as they only enforce that the label is not $-$. 
To address both of these concerns, 
the {\rm SUP5} and {\rm SUP6} constraints may be more desirable as they do not force an arc to hold between a pair or arguments, and if there is an arc, they only enforce that the label is not $-$.

\begin{example}
\label{ex:logicalmixed}
Returning to Example \ref{ex:saltsweetagain}, 
{\rm SUP3}, {\rm SUP4}, {\rm SUP5}, and {\rm SUP6}, hold.
However, {\rm SUP1} and {\rm SUP2} do not hold 
because $\qw{sa} \vee \qw{sw} \in \support(\inst(A))$ 
and $\claim(\inst(B))\vdash\qw{sa} \vee \qw{sw}$
and $\claim(\inst(C))\vdash\qw{sa} \vee \qw{sw}$
but $\lab(B,A) = \; *$ and $\lab(C,A) = \; *$.
\end{example}

\begin{example}
\label{ex:logicalmixed2}
Consider the following instantiated bipolar argument graph. 
This graph satisfies {\rm SUP3}, {\rm SUP4}, {\rm SUP5}, and {\rm SUP6},
but not {\rm SUP1} nor {\rm SUP 2}.
Also it satisfies {\rm INC2} and {\rm INC3}, but not {\rm INC1}. 
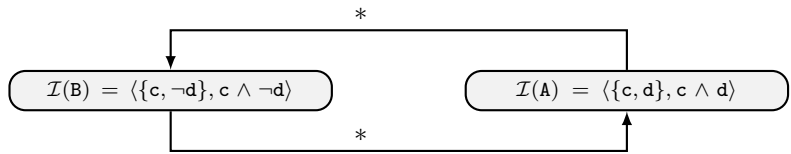
\begin{center}
\begin{tikzpicture}
[->,>=latex,thick, arg/.style={draw,text centered, 
shape=rectangle, rounded corners=6pt,
fill=gray!10,font=\footnotesize}]
\node[arg] (a) [text width=40mm] at (6,0) {$\inst(\xa) = \langle\{\tt c, d \}, c \wedge d \rangle$};
\node[arg] (b) [text width=40mm] at (0,0)  {$\inst(\xb) = \langle\{\tt c, \neg d \}, c \wedge \neg d \rangle$};
\draw[->] (a) |-  (2.5,0.8) node[above] {$*$} -| (b) ;
\draw[->] (b) |-  (2.5,-0.8) node[above] {$*$} -| (a) ;
\end{tikzpicture}
\end{center}
\end{example}

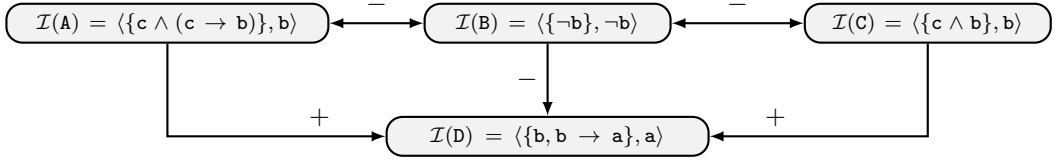
\begin{figure}
\begin{center}
\begin{tikzpicture}
[->,>=latex,thick, arg/.style={draw,text centered, 
shape=rectangle, rounded corners=6pt,
fill=gray!10,font=\footnotesize}]
\node[arg] (a) [text width=40mm] at (0,0) {$\inst(\xa) = \langle \tt \{ c \wedge (c \rightarrow b) \}, b \rangle $};
\node[arg] (b) [text width=30mm] at (5,0) {$\inst(\xb) = \langle \tt \{ \neg b \}, \neg b \rangle $};
\node[arg] (c) [text width=30mm] at (10,0) {$\inst(\xc) = \langle \tt \{ c \wedge b \}, b \rangle $};
\node[arg] (d) [text width=40mm] at (5,-1.5) {$\inst(\xd) = \langle \tt \{ b, b \rightarrow a \}, a \rangle $};
\path	(b) edge[<->] node[above] {$-$} (a);
\path	(c) edge[<->] node[above] {$-$} (b);
\path	(b) edge[->] node[left] {$-$} (d);
\draw[->] (a) |-  (2,-1.5) node[above] {$+$} -- (d) ;
\draw[->] (c) |-  (8,-1.5) node[above] {$+$} -- (d) ;
\end{tikzpicture}
\end{center}
\caption{\label{fig:satisfy}Example of an instantiated bipolar argument graph that satisfies the properties of {\rm EQUIV}, 
{\rm NEG1} to {\rm NEG7}, 
{\rm  POS1} to {\rm  POS4}, 
{\rm  POS6} to {\rm  POS9},
{\rm  INC1} to {\rm  INC3}, 
and {\rm  SUP1} to {\rm  SUP6}. }
\end{figure}

With the structural constraints, we have a number of choices. However, the essential considerations when choosing which constraints to adopt is whether we want to force an arc to hold between a pair of arguments in the graph, and whether we want to constraint what the label is for that arc.  

In general, we do not propose an absolute minimum set of structural constraints as there may be applications where any of the constraints could be deemed essential or inappropriate. However, it likely that some constraints would be adopted for most applications such as EQUIV, NEG1, NEG2, NEG3, POS1, and POS2.

\subsubsection{Properties of structural constraints}

Now we consider some inter-relationships and properties of the five groups of structural constraints presented in the previous subsection.

The constraints {\rm NEG1} to {\rm NEG3} are important for considering counterarguments for deductive arguments. 
The following result shows which constraints are satisfied for an instantiated bipolar argument graph according to the type of counterargument.

\begin{proposition}
For an instantiated bipolar argument graph $(\graph,\lab,\inst)$:
(1) If for all  $(A,B) \in \arcs(\graph)$, 
$\inst(A)$ is a defeater for $\inst(B)$, 
then $(\graph,\lab,\inst)$ satisfies {\rm NEG3}, but not {\rm NEG1}, nor {\rm NEG2}; 
(2) If for all  $(A,B) \in \arcs(\graph)$, 
$\inst(A)$ is an undercut for $\inst(B)$, 
then $(\graph,\lab,\inst)$ satisfies {\rm NEG3}, but not {\rm NEG1}, nor {\rm NEG2}; 
(3) If for all  $(A,B) \in \arcs(\graph)$, 
$\inst(A)$ is a direct undercut for $\inst(B)$, 
then $(\graph,\lab,\inst)$ satisfies {\rm NEG1}, {\rm NEG2}, and {\rm NEG3};  
(4) If for all  $(A,B) \in \arcs(\graph)$, 
$\inst(A)$ is a canonical undercut for $\inst(B)$, 
then $(\graph,\lab,\inst)$ satisfies {\rm NEG1}, {\rm NEG2}, and {\rm NEG3}; 
And 
(5)  If for all  $(A,B) \in \arcs(\graph)$, 
$\inst(A)$ is a defeating rebuttal for $\inst(B)$, 
then $(\graph,\lab,\inst)$ satisfies {\rm NEG1}, and {\rm NEG3}, but not {\rm NEG2}. 
\end{proposition}

\begin{proof}
(Defeater NEG1) Counterexample. 
$\inst(\qw{A}) = \langle\{\qw{a}\wedge\qw{b}\},\qw{a}\wedge\qw{b}\rangle$ 
is a defeater of $\inst(\qw{B}) = \langle\{\neg\qw{a},\neg\qw{a}\rightarrow\qw{b}\},\qw{b}\rangle$, 
but $\claim(\inst(\qw{A}))\vdash\claim(\inst(\qw{B}))$,
and so ${\rm NEG1}$ fails.
(Defeater NEG2) Counterexample. 
$\inst(\qw{A}) = \langle\{\qw{a}\wedge\qw{b}\},\qw{a}\wedge\qw{b}\rangle$ 
is a defeater of $\inst(\qw{B}) = \langle\{\neg\qw{a},\qw{b}\},\neg\qw{a}\wedge\qw{b}\rangle$, 
but $\claim(\inst(\qw{A}))\vdash\qw{b}$,
and so ${\rm NEG2}$ fails.
(Defeater NEG3) Assume for all  $(A,B) \in \arcs(\graph)$, 
$\inst(A)$ is a defeater for $\inst(B)$, 
and so $\claim(\inst(A))\vdash\neg\wedge\support(\inst(B))$.
So for all  $(A,B) \in \arcs(\graph)$, 
$\claim(\inst(A))\cup\support(\inst(B))\vdash\bot$.
So ${\rm NEG3}$ succeeds.
(Undercut NEG1) 
 Counterexample. 
$\inst(\qw{A}) = \langle\{\neg\qw{a}\wedge\qw{b}\},\neg\qw{a}\wedge\qw{b}\rangle$ 
is an undercut of $\inst(\qw{B}) = \langle\{\qw{a},\qw{a}\rightarrow\qw{b}\},\qw{b}\rangle$, 
but $\claim(\inst(\qw{A}))\vdash\claim(\inst(\qw{B}))$,
and so ${\rm NEG1}$ fails.
(Undercut NEG2) 
 Counterexample. 
$\inst(\qw{A}) = \langle\{\neg\qw{a}\wedge\qw{b}\},\neg\qw{a}\wedge\qw{b}\rangle$ 
is an undercut of $\inst(\qw{B}) = \langle\{\qw{a},\qw{b}\},\qw{a}\wedge\qw{b}\rangle$, 
but $\claim(\inst(\qw{A}))\vdash\qw{b}$,
and so ${\rm NEG2}$ fails.
(Undercut NEG3)
Assume for all  $(A,B) \in \arcs(\graph)$, 
$\inst(A)$ is an undercut for $\inst(B)$, 
and so there exists a $\Psi \subseteq {\support}(\inst(B))$ 
s.t. ${\claim}(\inst(A)) \equiv \neg \wedge \Psi$.
So for all  $(A,B) \in \arcs(\graph)$, 
$\claim(\inst(A))\cup\support(\inst(B))\vdash\bot$.
So ${\rm NEG3}$ succeeds.
(Direct undercut NEG1)
Assume for all  $(A,B) \in \arcs(\graph)$, 
$\inst(A)$ is a direct undercut for $\inst(B)$.
So for each $(A,B)$, 
${\claim}(\inst(A)) \equiv \neg\phi$ for some $\phi\in\support(\inst(B))$.
Recall that a deductive argument is such that $\support(\inst(A))\not\vdash\bot$ and $\support(\inst(B))\not\vdash\bot$.
So it is not possible that also $\claim(\inst(A))\vdash\claim(\inst(B))$ holds.
So ${\rm NEG1}$ succeeds.
(Direct undercut NEG2)
Assume for all  $(A,B) \in \arcs(\graph)$, 
$\inst(A)$ is a direct undercut for $\inst(B)$.
So for each $(A,B)$, 
${\claim}(\inst(A)) \equiv \neg\phi$ for some $\phi\in\support(\inst(B))$.
Recall that a deductive argument is such that $\support(\inst(A))\not\vdash\bot$ and $\support(\inst(B))\not\vdash\bot$.
So it is not possible that also $\claim(\inst(A))\vdash\psi$ 
for some $\psi\in\support(\inst(B))$.
So ${\rm NEG2}$ succeeds.
(Direct undercut NEG3)
Assume for all  $(A,B) \in \arcs(\graph)$, 
$\inst(A)$ is a direct undercut for $\inst(B)$, 
and so there exists a $\phi \in {\support}(\inst(B))$ 
s.t. ${\claim}(\inst(A)) \equiv \neg \phi$.
So for all  $(A,B) \in \arcs(\graph)$, 
$\claim(\inst(A))\cup\support(\inst(B))\vdash\bot$.
So ${\rm NEG3}$ succeeds.
(Canonical undercut NEG1)
Assume for all  $(A,B) \in \arcs(\graph)$, 
$\inst(A)$ is a canonical undercut for $\inst(B)$.
So for each $(A,B)$, 
${\claim}(\inst(A)) \equiv \neg (\phi_1\wedge\ldots\wedge\phi_n)$ 
for $\support(\inst(B)) = \{\phi_1,\ldots,\phi_n\}$.
Recall that a deductive argument is such that $\support(\inst(A))\not\vdash\bot$ and $\support(\inst(B))\not\vdash\bot$.
So it is not possible that also $\claim(\inst(A))\vdash\claim(\inst(B))$ holds.
So ${\rm NEG1}$ succeeds.
(Canonical undercut NEG2)
Assume for all  $(A,B) \in \arcs(\graph)$, 
$\inst(A)$ is a canonical undercut for $\inst(B)$.
So for each $(A,B)$, 
${\claim}(\inst(A)) \equiv \neg (\phi_1\vee\ldots\vee\phi_n)$ 
for $\support(\inst(B)) = \{\phi_1,\ldots,\phi_n\}$.
So there is no $\phi_i\in \support(\inst(B))$
s.t. $\claim(\inst(A))\vdash\phi_i$ holds, 
otherwise by resolution, 
${\claim}(\inst(A)) \equiv \neg \vee (\support(\inst(B))\setminus\{\phi_i\}$
and hence $\inst(A)$ would not be a canonical undercut for $\inst(B)$.
(Canonical undercut NEG3)
So for each $(A,B)$, 
${\claim}(\inst(A)) \equiv \neg (\phi_1\vee\ldots\vee\phi_n)$ 
for $\support(\inst(B)) = \{\phi_1,\ldots,\phi_n\}$.
So for all  $(A,B) \in \arcs(\graph)$, 
$\claim(\inst(A))\cup\support(\inst(B))\vdash\bot$.
So ${\rm NEG3}$ succeeds.
(Defeating rebuttal NEG1)
Assume for all  $(A,B) \in \arcs(\graph)$, 
$\inst(A)$ is a defeating rebuttal for $\inst(B)$.
So for each $(A,B)$, 
${\claim}(\inst(A)) \vdash \neg {\claim}(\inst(B))$.
Recall that a deductive argument is such that $\support(\inst(A))\not\vdash\bot$ and $\support(\inst(B))\not\vdash\bot$.
So it is not possible that also $\claim(\inst(A))\vdash\claim(\inst(B))$ holds.
So ${\rm NEG1}$ succeeds.
(Defeating rebuttal NEG2) Counterexample. 
$\inst(\qw{A}) = \langle\{\qw{a}\wedge\neg\qw{b}\},\qw{a}\wedge\neg\qw{b}\rangle$ 
is a defeating rebuttal of $\inst(\qw{B}) = \langle\{\qw{a},\qw{a}\rightarrow\qw{b}\},\qw{b}\rangle$, 
but $\claim(\inst(\qw{A}))\vdash\qw{a})$,
and so ${\rm NEG2}$ fails.
(Defeating rebuttal NEG3) 
Assume for all  $(A,B) \in \arcs(\graph)$, 
$\inst(A)$ is a defeating rebuttal for $\inst(B)$,  
So for each $(A,B)$, 
${\claim}(\inst(A)) \vdash \neg {\claim}(\inst(B))$.
So, 
$\claim(\inst(A))\cup\support(\inst(B))\vdash\bot$.
So ${\rm NEG3}$ succeeds.
\end{proof}

With the NEG constraints, we have a hierarchy in terms of generality (i.e. which implies which) is captured by the following result. 

\begin{proposition}
(1) {\rm NEG1} does not imply {\rm NEG2} nor vice versa.
(2) {\rm NEG1} does not imply {\rm NEG3} nor vice versa.
(3) {\rm NEG2} does not imply {\rm NEG3} nor vice versa.
(4) {\rm NEG3} does not imply any of {\rm NEG4}, {\rm NEG5}, {\rm NEG6}, 
or {\rm NEG7}, or vice versa.
(5) {\rm NEG4} does not imply {\rm NEG5} or vice versa.
(6) {\rm NEG4} implies {\rm NEG6} but the converse does not hold. 
(7) {\rm NEG5} implies {\rm NEG7} but the converse does not hold. 
(8) {\rm NEG6} does not imply {\rm NEG7} or vice versa.
\end{proposition}

\begin{proof}
(1) For a counterexample, 
consider $\tt \langle\{b,b\rightarrow a\},a \wedge b \rangle$ 
and $\tt \langle\{b,b\rightarrow \neg a\},\neg a\rangle$,
and so {\rm NEG1} holds but {\rm NEG2} does not hold.
For the failure of the converse,
consider $\tt \langle\{a \wedge b\},a\wedge b \rangle$ 
and $\tt \langle\{\neg a,\neg a \rightarrow b \}, b\rangle$,
and so {\rm NEG2} holds but {\rm NEG1} does not hold.
(2) For a counterexample, 
consider $\tt \langle\{a\},a\rangle$ 
and $\tt \langle\{b\},b\rangle$,
and so {\rm NEG1} holds but {\rm NEG3} does not hold.
For the failure of the converse,
consider $\tt \langle\{a\wedge b\},a\wedge b \rangle$ 
and $\tt \langle\{\neg a,\neg a \rightarrow b \}, b\rangle$,
and so {\rm NEG3} holds but {\rm NEG1} does not hold.
(3) For a counterexample, 
consider $\tt \langle\{a\},a\rangle$ 
and $\tt \langle\{b\},b\rangle$,
and so {\rm NEG2} holds but {\rm NEG3} does not hold.
For the failure of the converse,
consider $\tt \langle\{a\wedge b\},a\wedge b \rangle$ 
and $\tt \langle\{\neg a,\neg a \rightarrow b \}, b\rangle$,
and so {\rm NEG3} holds but {\rm NEG2} does not hold.
(4) Direct from definition.
(5) Direct from definition.
(6) Assume $\lab(A, B) = -$ 
and $\claim(\inst(C)) \vdash \claim(\inst(A))$ 
and $\lab(C, B) = -$ hold.
So $(C,B)\in\arcs(\graph)$  and $\lab(C, B) \neq +$ hold.
Hence {\rm NEG4} implies {\rm NEG6}.
For the converse, if $(C,B)\not\in\arcs(\graph)$, or $\lab(C, B) \neq \; *$,
then {\rm NEG6} holds but {\rm NEG4} does not hold. 
(7) Assume $\lab(A, B) = -$ 
and $\support(\inst(B)) \subseteq \support(\inst(C))$, 
then $\lab(A, C) = -$. 
So $(A,C)\in\arcs(\graph)$, then $\lab(A, C) \neq +$ hold.
Hence {\rm NEG5} implies {\rm NEG7}.
For the converse, if $(A,C)\not\in\arcs(\graph)$, 
or $\lab(A, C) \neq \; *$,
then {\rm NEG7} holds but {\rm NEG5} does not hold. 
(8) Direct from definition.
\end{proof}

We also have a hierarchy of POS constraints in terms of generality (i.e. which implies which) as captured by the following result.

\begin{proposition}
(1) {\rm  POS1} implies {\rm  POS2} but the converse does not hold.
(2) {\rm  POS3} implies {\rm  POS4} but not vice versa.
(3) {\rm  POS3} does not imply {\rm  POS5} or vice versa.
(4) Neither {\rm  POS3} nor {\rm POS4} imply any of {\rm  POS6}, {\rm  POS7}, {\rm  POS8} or {\rm  POS9}, or vice versa.
(5) {\rm  POS5} implies {\rm  POS2} but not vice versa.
(6) {\rm  POS6} implies {\rm  POS8} but the converse does not hold. 
(7) {\rm  POS6} does not imply {\rm  POS7} or vice versa.
(8) {\rm  POS7} implies {\rm  POS9} but the converse does not hold. 
(9) {\rm  POS8} does not imply {\rm  POS9} or vice versa.
\end{proposition}

\begin{proof}
(1) Assume $\support(\inst(A)) \cup \support(\inst(B))\not\vdash\bot$ holds.
So $\{ \claim(\inst(A)) \} \cup \support(\inst(B))\not\vdash\bot$ holds. 
So {\rm  POS1} implies {\rm  POS2}.
For the failure of the converse, 
consider $\tt A = \langle\{\tt b,b\rightarrow a\},a \rangle$ 
and $\tt B = \langle\{a,a\rightarrow \neg b\}, \neg b \rangle$,
and so {\rm  POS2} holds but not {\rm  POS1}. 
(2) Assume there is a $\alpha \in \support(\inst(B))  \mbox{ s.t. } \claim(\inst(A))$ is $\alpha$.
Now let $\Gamma = \{ \alpha\}$.
So there is a $\Gamma\subseteq \support(\inst(B))$ s.t. $\claim(\inst(A))$ is $\wedge\Gamma$.
So {\rm  POS3} implies {\rm  POS4}.
To show failure of converse, consider $\Gamma$ with two or more formulae and $\claim(\inst(A)) = \wedge\Gamma$.
So {\rm  POS4} holds but not {\rm  POS3}. 
(3) Follows from definitions.
(4) Follow directly from definitions.
(5) Assume $\claim(\inst(A)) \vdash \wedge\support(\inst(B))$.
By definition of an argument, $\support(\inst(A)\not\vdash\bot$.
So $\claim(\inst(A)\not\vdash\bot$.
So $\{ \claim(\inst(A)) \} \cup \support(\inst(B))\not\vdash\bot$.
So {\rm  POS5} implies {\rm  POS2}.
To show failure of converse, 
consider $A = \langle\{\tt a\},a\rangle$ 
and $B = \langle\{\tt b\}, b\rangle$.
So {\rm  POS2} holds but {\rm  POS5} does not hold. 
(6) Assume {\rm  POS6} holds.
So if $\lab(A, B) = +$,
and $\claim(\inst(C)) \vdash \claim(\inst(A))$, 
then $\lab(C, B) = +$. 
So if $\lab(A, B) = +$, 
and $(C,B)\in\arcs(\graph)$, 
and $\claim(\inst(C)) \vdash \claim(\inst(A))$, 
then $\lab(C, B) \neq -$. 
So {\rm  POS6} implies {\rm  POS8}.
To show the converse does not hold consider either $(C,B)\not\in\arcs(\graph)$
or $\lab(C, B) = \; *$. 
In either case, {\rm  POS8} holds but {\rm  POS6} does not hold. 
(7) Follow directly from definitions. 
(8) Assume {\rm  POS7}.
So if $\lab(A, B) = +$ and $\support(I(B)) = \support(I(C))$ then $L(A, C) = +$.
So if $\lab(A, B) = +$,
and $(A,C)\in\arcs(\graph)$, 
and $\support(\inst(B)) \subseteq \support(\inst(C))$, 
then $\lab(A, C) \neq -$.
So {\rm  POS7} implies {\rm  POS9}. 
To show the converse does not hold consider either $(C,B)\not\in\arcs(\graph)$
or $\lab(C, B) = \; *.$
In either case, {\rm POS9} holds but {\rm  POS7} does not holds. 
(9) Follow directly from definitions. 
\end{proof}

Of the INC constraints, it is straightforward to see that there is a linear hierarchy where INC3 is the most general,
and INC1 is the least general. 

\begin{proposition}
{\em INC1} implies {\em INC2} and {\em INC2} implies {\em INC3}.
\end{proposition}

\begin{proof}
Follows directly from the definitions.
\end{proof}

\begin{proposition}
If instantiated bipolar argument graph $\igraph = (\graph,\lab,\inst)$ satisfies {\rm  INC1}, 
and for all arcs $(A,B) \in \arcs(\graph)$, $\lab(A,B) = +$,
and $\igraph$ exhausts $\Delta$ , 
then $\Delta$ is consistent. 
\end{proposition}

\begin{proof}
Assume $\igraph$ exhausts $\Delta$.
Therefore, for all $I \in \args(\Delta)$,
there is an $A \in \nodes(\graph)$ s.t. $\inst(A) = I$.
Furthermore, $\igraph$ satisfies {\rm  INC1},
and for all arcs $(A,B) \in \arcs(\graph)$, $\lab(A,B) = +$.
Therefore, there are no $A,B \in \nodes(\Delta)$, 
such that $\{ \claim(\inst(A)) \} \cup \support(\inst(B)) \vdash \bot$.
Therefore, $\Delta$ is consistent. 
\end{proof}

Finally, for the SUP constraints, it is also straightforward to see that we have a hierarchy as captured in the following result. 

\begin{proposition}
(1) {\rm  SUP1} implies {\rm  SUP3} but the converse does not hold.
(2) {\rm  SUP2} implies {\rm  SUP1} but the converse does not hold.
(3) {\rm  SUP2} implies {\rm  SUP4} but the converse does not hold.
(4) {\rm  SUP3} implies {\rm  SUP5} but the converse does not hold.
(5) {\rm  SUP4} implies {\rm  SUP6} but the converse does not hold.
(6) {\rm  SUP4} implies {\rm  SUP3} but the converse does not hold.
(7) {\rm  SUP6} implies {\rm  SUP5} but the converse does not hold.
\end{proposition}

\begin{proof}
Follows directly from the definitions.
\end{proof}

\begin{proposition}
The set of constraints {\rm EQUIV}, {\rm NEG1} to {\rm NEG7}, 
{\rm  POS1} to {\rm  POS9},
{\rm  INC1} to {\rm  INC3}, 
and {\rm  SUP1} to {\rm  SUP6} are consistent together. 
\end{proposition}

\begin{proof}
To show consistency, consider the following bipolar argument graph which satisfies all the constraints. 
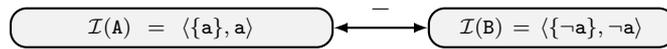
\begin{center}
\begin{tikzpicture}
[->,>=latex,thick, arg/.style={draw,text centered, 
shape=rectangle, rounded corners=6pt,
fill=gray!10,font=\footnotesize}]
\node[arg] (a) [text width=40mm] at (0,0) {$\inst(\xa) = \langle \tt \{ a \}, a \rangle $};
\node[arg] (b) [text width=30mm] at (5,0) {$\inst(\xb) = \langle \tt \{ \neg a \}, \neg a \rangle $};
\path	(b) edge[<->] node[above] {$-$} (a);
\end{tikzpicture}
\end{center}
\end{proof}

The above result shows that there are situations where all the properties can hold. However, there are also situations where some of the properties cannot hold together. For example, for many instantiations it is not possible for both {\rm  POS3} and POS6 to both hold or for both {\rm NEG4} and {\rm NEG5} to hold. This is particularly the case when all the deductive arguments in $\args(\Delta)$ appear in the instantiated bipolar argument graph for a given knowledgebase $\Delta$. 

\section{Discussion}
\label{section:discussion}

In this paper, we have proposed a framework for instantiating bipolar argument graphs with deductive arguments. This has included a set of constraints for delineating acceptable instantiations. The focus has been on instantiation with deductive arguments where the base logic is classical logic. But this could be adapted for alternative base logics (see \cite{Hunter2010,BesnardHunter2014,Besnard2018} for discussion of alternative base logics).  

In future work, we will consider specific proposals for instantiations of bipolar argument graphs, and specific labelling policies, and investigate the constraints that are satisfied. We will also consider definitions for comparing instantiations (e.g. for saying that two instantiations of the same bipolar argument graph are logically equivalent, or that one instantiation is more specific than another instantiation) and definitions for manipulating instantiations (e.g. combining or splitting nodes, and therefore their instantiations). This may involve establishing relationships with notions of similarity proposed for bipolar argumentation \cite{Budan2020}.

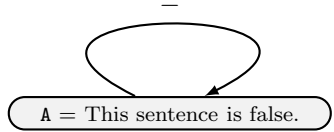
\begin{figure}
\begin{center}
\begin{tikzpicture}
[->,>=latex,thick, arg/.style={draw,text centered, 
shape=rectangle, rounded corners=6pt,
fill=gray!10,font=\footnotesize}]
\node[arg] (a) [text width=40mm] at (3,0) {$\xa$ = {\rm This sentence is false.}};
\draw (a) .. controls (0,1.5) and (6,1.5) .. (a) node[above,pos=0.5]{$-$};
\end{tikzpicture}
\end{center}
\caption{\label{ex:selfcycle}
Bipolar argument graph with a single self-attacking arc.
}
\end{figure}

Also in future work, we will consider generalizing the notion of an instantiation function so that it can assign approximate arguments to nodes. 
An {\bf approximate argument} is a pair $\langle\Phi,\alpha\rangle$ 
where $\Phi\subseteq{\cal L}$ and $\alpha \in {\cal L}$ \cite{Hunter2007}.
This is a very general definition. 
It does not assume that $\Phi$ is consistent, or that it even entails $\alpha$.
Consider Figure \ref{ex:selfcycle}, if we allow approximate arguments, we could for instance instantiate as follows.
\[
\inst(\xa) = \langle \tt \{ \neg a, \neg a \rightarrow a \}, a \rangle 
\]

By introducing the flexibility to instantiate bipolar argument graphs with approximate arguments, it will allow us to formalize enthymemes 
which are arguments where there are insufficient premises for entailing the claim and/or the claim is implicit or incomplete \cite{Black:2012}.
Advantages of instantiating bipolar argument graphs with enthymemes in the form of approximate arguments is that we can allow supporting arguments to provide implicit premises, and thereby avoid repetition of information that is presented to the user. 
Also, by using bipolar argument graphs to understand enthymemes, the uncertainty involved with the decoding can be quantified and managed (for example, by generalizing the constellation approach to probabilistic argumentation \cite{Hunter2014}).

Finally, in future work, we want to investigate how we can instantiate an argument in an bipolar argument graph with alternative arguments. This may allow us to have ways of reasoning directly with {\bf ambiguity} where there are different deductive arguments, perhaps with quite different non-logical symbols (i.e different symbols for propositions, predicates, and terms) reflecting quite different interpretations of the abstract argument, {\bf granularity} where an abstract argument can be instantiated with premises being a few propositional formulae or with premises being a large number of very complex predicate logic formulae, and {\bf veracity} where there may be different instantiations depending on whether the argument is regarded as acceptable and why (for example, for an argument, we may instantiate it differently when the argument is believed or disbelieved).


\newcommand{\etalchar}[1]{$^{#1}$}


\end{document}